\documentclass[11pt]{article}
\usepackage{algorithm}
\usepackage{algpseudocode}
\usepackage{booktabs}
\usepackage{multirow}
\usepackage{graphicx}
\usepackage{svg}
\usepackage{wrapfig}
\usepackage[margin=1in]{geometry}
\usepackage{amsmath}
\usepackage{caption}
\usepackage{subcaption}
\usepackage{amssymb}
\usepackage{amsthm}
\usepackage{caption}
\usepackage{algorithm}
\usepackage{algpseudocode}
\newtheorem{theorem}{Theorem}
\newtheorem{lemma}{Lemma}

\usepackage{xcolor}
\usepackage[numbers,sort&compress]{natbib}

\usepackage{authblk}

\usepackage{enumitem}
\usepackage{float} 
\AtBeginDocument{%
  }

\algrenewcommand{\algorithmiccomment}[1]{\hfill\textcolor{gray}{\(\triangleright\) \textit{#1}}}
\author[1]{Haosen Li$^{*}$}
\author[1]{Wenshuo Chen$^{*, \ddagger}$}
\author[2]{Lei Wang}
\author[1]{Shaofeng Liang}
\author[1]{Haozhe Jia}
\author[1]{Yutao Yue$^{\dagger}$}

\affil[1]{The Hong Kong University of Science and Technology (Guangzhou)}
\affil[2]{Griffith University \& Data61/CSIRO}
\affil[]{\small
$^{*}$Equal contribution.\quad
$^{\ddagger}$ Project Leader \quad
$^{\dagger}$Corresponding author\par
yutaoyue@hkust-gz.edu.cn
}

\date{}

\begin{document}
\title{Oracle Noise: Faster Semantic Spherical Alignment for Interpretable Latent Optimization}

\maketitle
\vspace{-1cm}
\begin{figure}[H]
  \centering
  \includegraphics[width=0.9\textwidth]{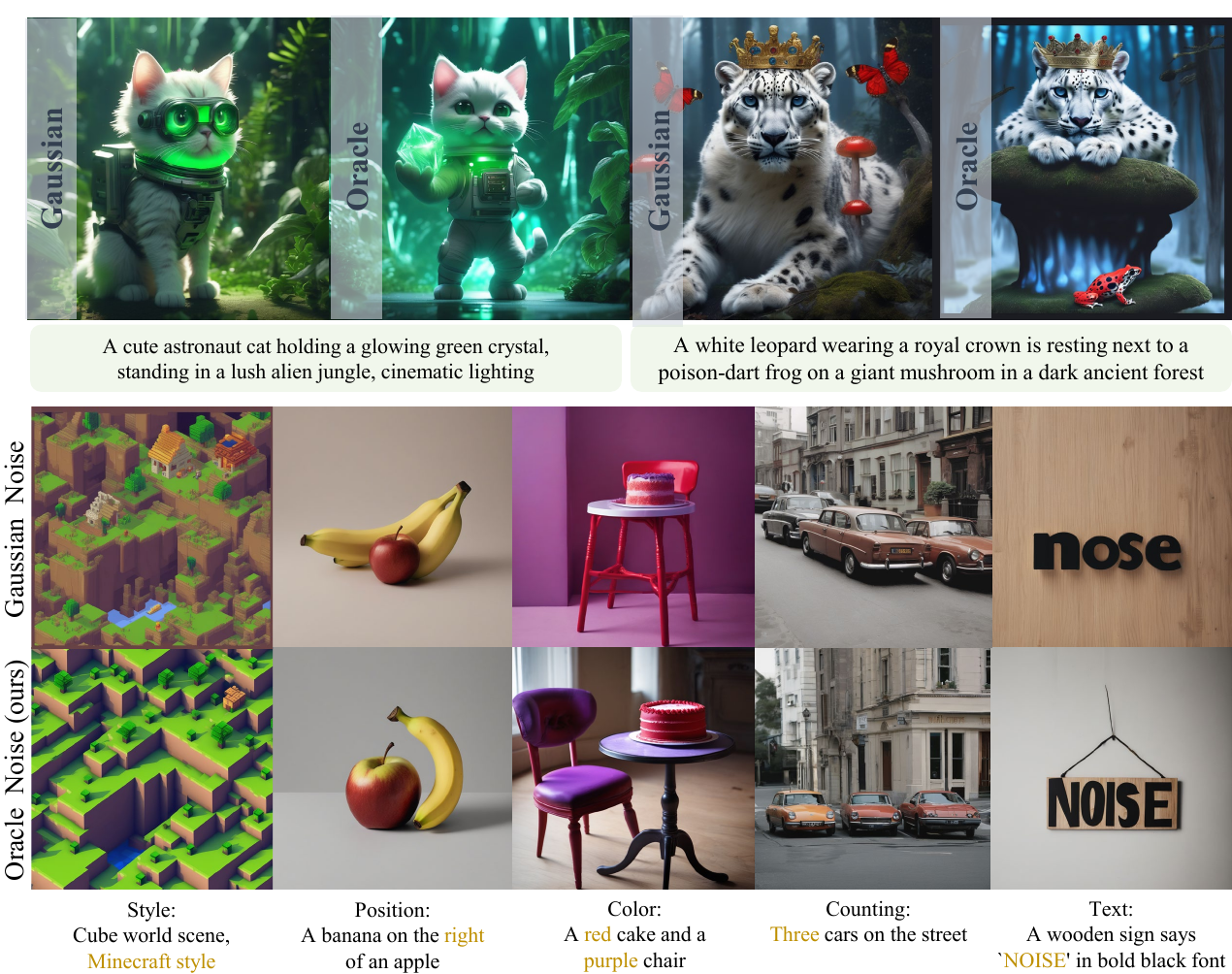}
  \caption{\textbf{Qualitative comparison of Oracle Noise vs. Gaussian Noise.} \textbf{Top:} For complex prompts, our method accurately renders intricate details and multiple object compositions where the baseline struggles. \textbf{Bottom:} Our method demonstrates superior detailed control across five key attributes: style adherence, spatial positioning, multiple color binding, accurate counting, and explicit text rendering. Overall, Oracle Noise significantly improves prompt adherence and generation fidelity.}
  \label{fig:teaser}
\end{figure}

\begin{abstract}

Text-to-image diffusion models have achieved remarkable generative capabilities, yet accurately aligning complex textual prompts with synthesized layouts remains an ongoing challenge. In these models, the initial Gaussian noise acts as a critical structural seed dictating the macroscopic layout. Recent online optimization and search methods attempt to refine this noise to enhance text-image alignment. However, relying on unconstrained Euclidean gradient ascent mathematically inflates the latent norm and destroys the standard Gaussian prior, causing severe visual artifacts like color over-saturation. Furthermore, these methods suffer from inefficient semantic routing and easily fall into the ``reward hacking'' trap of external proxy models. To address these intertwined bottlenecks, we propose \textbf{Oracle Noise}, a zero-shot framework reframing noise initialization as semantic-driven optimization strictly confined to a Riemannian hypersphere. Instead of relying on complex external parsers, we directly identify the most impactful structural words in the prompt to efficiently route optimization energy. By updating the noise strictly along a spherical path, we mathematically preserve the original Gaussian distribution. This geometric constraint eliminates norm inflation and unlocks aggressive step sizes for rapid convergence. Extensive experiments demonstrate that Oracle Noise significantly accelerates semantic alignment and achieves superior aesthetics without black-box models. It completely mitigates Euclidean-induced degradation, establishing state-of-the-art performance across human preference metrics (e.g., HPSv2, ImageReward), semantic alignment (CLIP Score), and sample diversity, all within a strict 2-second optimization budget.

\end{abstract}

\section{Introduction}
\label{sec:intro}

Diffusion models ~\cite{ho2020denoisingdiffusionprobabilisticmodels, song2021scorebasedgenerativemodelingstochastic, song2022denoisingdiffusionimplicitmodels, lipman2023flowmatchinggenerativemodeling, liu2022flowstraightfastlearning, rombach2022highresolutionimagesynthesislatent, podell2023sdxlimprovinglatentdiffusion, chen_2024, Chen_2025, chen2025polarisprojectionorthogonalsquaresrobust, yuan2025coemogensemanticallycoherentscalableemotional, ning2025dctdiffintriguingpropertiesimage, jia2025physicsinformedrepresentationalignmentsparse}
have fundamentally transformed the landscape of visual synthesis, offering an unprecedented medium to materialize human imagination into high-fidelity imagery. At the algorithmic core of these models lies the process of sculpting order out of absolute chaos, where every generation trajectory begins with a blank canvas of pure randomness. Recently, it has become a consensus within the community that, far from being a mere algorithmic placeholder, this initial noise acts as the critical structural seed dictating the macroscopic layout, entity placement, and eventual semantic composition of the synthesized image \cite{tang2024zerothorderoptimizationmeetshuman, 10657632, ahn2024noiseworthdiffusionguidance, huang2024bluenoisediffusionmodels, zhou2025goldennoisediffusionmodels, ma2025inferencetimescalingdiffusionmodels, qi2024noisescreatedequallydiffusionnoise, eyring2024renoenhancingonesteptexttoimage, wang2025silentassistantnoisequeryimplicit, jia2026antitheticnoisediffusionmodels} . Consequently, ensuring that this foundational canvas perfectly resonates with the user's textual intent before the generation begins is a crucial step for achieving flawless text-to-image alignment.

In modern Latent Diffusion Models (LDMs) \cite{rombach2022highresolutionimagesynthesislatent}, this initial canvas takes the mathematical form of a high-dimensional latent variable drawn from a standard Gaussian prior \cite{ho2020denoisingdiffusionprobabilisticmodels} . While earlier online optimization and search methods often focused on adjusting intermediate generative trajectories or text embeddings, recent works have shifted to directly targeting this initial noise \cite{ma2025inferencetimescalingdiffusionmodels, 10657632, qi2024noisescreatedequallydiffusionnoise, eyring2024renoenhancingonesteptexttoimage}, attempting to proactively refine it prior to the reverse diffusion process. Despite their conceptual potential, these noise-centric approaches remain severely bottlenecked by fundamental and intertwined limitations. In practice, existing methods predominantly rely on gradient ascent within an unconstrained Euclidean space. However, in high-dimensional spaces, Gaussian distributions are densely concentrated on a hyperspherical shell \cite{davidson2022hypersphericalvariationalautoencoders,dosi2025harmonizinggeometryuncertaintydiffusion, Vershynin_2018}. Consequently, Euclidean optimization inherently inflates the vector norm of the latent representation, causing it to deviate from its native prior. This geometric shift forces the latent variable into out-of-distribution regions, empirically manifesting as structural degradation and color over-saturation. 

Compounding this geometric degradation is a fundamentally flawed text-attention paradigm in current semantic routing. We observe that existing approaches treat every textual token equally, allocating the exact same optimization bandwidth to non-informative functional words as they do to core semantic entities. This strictly equal weighting disperses the model's generation capacity, readily leading to severe semantic misalignment or even completely erroneous visual layouts~\cite{10657632}. Furthermore, correcting this dispersion necessitates numerous computationally heavy backpropagation steps, substantially slowing convergence. To mitigate the resulting drop in generation quality, recent pipelines often incorporate external human-preference reward models as auxiliary objectives~\cite{eyring2024renoenhancingonesteptexttoimage}. While occasionally effective, relying on these proxy models inevitably increases memory consumption and introduces a critical vulnerability to ``reward hacking''~\cite{eyring2024renoenhancingonesteptexttoimage, tang2024zerothorderoptimizationmeetshuman}. The optimization trajectory learns to exploit statistical biases in the proxy metrics rather than enhancing true visual fidelity, often leading to mode collapse and stripping the generative process of its interpretability.

To address these bottlenecks at their root, we propose \textbf{Oracle Noise Optimization}, a zero-shot, prior-preserving, and self-contained framework. We re-conceptualize the noise initialization problem by shifting from a flawed Euclidean search to a structurally rigorous optimization strictly confined to a Riemannian manifold. Specifically, our framework introduces a novel token weighting mechanism based on \textit{representational collapse}. Ideally, optimization should route energy strictly toward the tokens the user cares about most; however, obtaining explicit ground truth for such subjective preference is intractable at inference time. Recent insights from attention-based learning \cite{li2026reinforcedattentionlearning} reveal that superior generative alignment intrinsically stems from correct internal attention allocation, rather than superficial output optimization which easily triggers reward hacking. Motivated by this, we propose an intrinsic alternative: forcing the model to focus on tokens that mathematically drive the generation. By measuring the semantic shift in the native text embeddings when a token is masked, we unsupervisedly isolate the ``load-bearing'' structural words that possess the highest intrinsic generative value, completely bypassing the need for external syntactic parsers or proxy reward models. 

Subsequently, to eliminate geometric degradation, we formulate the rigorous mathematical proof of the Gaussian Annulus Theorem ~\cite{Vershynin_2018} . Based on this theoretical foundation, we execute a \textit{spherical geodesic update}: by orthogonally projecting the raw gradient onto the tangent space and navigating strictly along the hypersphere, our method mathematically guarantees that the Gaussian prior is preserved. This strict geometric constraint inherently prevents norm inflation, thereby unlocking aggressive step sizes for drastically accelerated convergence and completely eliminating the conditions that trigger reward hacking. Ultimately, Oracle Noise unlocks the potential of native pre-trained diffusion models, achieving superior semantic alignment and aesthetic quality without the burden of auxiliary reward models or Euclidean-induced visual degradation.
Our main contributions are three-fold:
\begin{itemize}[noitemsep, topsep=0pt, parsep=0pt, partopsep=0pt]
    \item We identify the fundamental theoretical flaws of existing online optimization and search methods, specifically Euclidean-induced geometric degradation, semantic misallocation, and the vulnerability to \textit{reward hacking} caused by reliance on external proxy models.
    
    \item We formulate a rigorous mathematical proof that noise initialization must be treated as a hyperspherical optimization problem. Based on this, we propose \textbf{Oracle Noise}, featuring a novel token weighting mechanism based on \textit{representational collapse} capturing intrinsic generative value, coupled with a spherical geodesic update that strictly preserves the Gaussian prior while enabling rapid convergence.
    
    \item Extensive experiments demonstrate that Oracle Noise establishes new SOTA performance among online noise optimization and search methods. It entirely eliminates the need for external reward models and mitigates Euclidean-induced artifacts, achieving superior semantic alignment and aesthetic quality within an extremely short inference time (a strict 2-second budget).
\end{itemize}

\section{Related Works}

\paragraph{Inference-Time Noise Optimization and Scaling.} 
While standard diffusion models draw the initial latent $x_T$ from a Gaussian prior, recent studies reveal this initialization heavily dictates macroscopic layout and semantic alignment. Consequently, instance-level optimization strategies have emerged. Zero-shot methods like ``The Silent Prompt'' \cite{wang2025silentassistantnoisequeryimplicit} and InitNO \cite{guo2024initnoboostingtexttoimagediffusion} enhance text alignment directly, whereas scaling approaches like ReNO \cite{eyring2024renoenhancingonesteptexttoimage} and Stable Noise \cite{qi2024noisescreatedequallydiffusionnoise} utilize external reward models for gradient feedback. However, these paradigms face critical limitations. Reward-guided methods introduce computational bottlenecks and are prone to ``reward hacking.'' More fundamentally, performing gradient ascent in an unconstrained Euclidean space inflates the vector norm, pushing latents off their native hyperspherical manifold and inducing severe geometric degradation and artifacts.

\paragraph{Learned Noise Priors and Open Challenges.} 
Beyond instance-specific optimization, frameworks like ``Golden Noise'' \cite{zhou2025goldennoisediffusionmodels} aim to map standard Gaussian noise into a superior, learned global prior. While effective, learning a new prior requires intensive additional training. Achieving optimal noise alignment in a zero-shot manner, while strictly preserving the pre-trained latent space's geometric properties, remains an open challenge. Our proposed Oracle Noise framework directly addresses this gap, offering a principled optimization strategy that achieves zero-shot semantic alignment without succumbing to manifold degradation or proxy metric exploitation.
\section{Preliminaries}
\label{sec:preliminaries}

\subsection{Diffusion Models and Cross-Attention}
\label{subsec:ldm_and_attention}

LDMs perform the generative process within a compressed latent space. A pre-trained encoder \cite{rombach2022highresolutionimagesynthesislatent} projects an image $x$ into a latent representation $z_0$. The forward diffusion process gradually adds Gaussian noise over $T$ timesteps, producing a final state $z_T \sim \mathcal{N}(\mathbf{0}, \mathbf{I})$ \cite{ho2020denoisingdiffusionprobabilisticmodels, song2022denoisingdiffusionimplicitmodels} . During inference, the generative process starts from randomly sampled pure noise $z_T$. A conditional denoising network $\epsilon_\theta(z_t, t, c)$ (e.g., U-Net \cite{dosi2025harmonizinggeometryuncertaintydiffusion} or Transformer \cite{peebles2023scalablediffusionmodelstransformers} ) is trained to iteratively denoise the latent guided by a text prompt $c$:
\begin{equation}
    \mathcal{L}_{\text{LDM}} = \mathbb{E}_{z \sim \mathcal{E}(x), \epsilon \sim \mathcal{N}(\mathbf{0}, \mathbf{I}), t} \left[ \|\epsilon - \epsilon_\theta(z_t, t, c)\|^2 \right]
\end{equation}
The denoising network injects the textual conditions $c$ into the spatial latent features $\phi(z_t)$ primarily through cross-attention layers. Using text embeddings $\tau(c)$ extracted by the text encoder \cite{radford2021learningtransferablevisualmodels} , the attention mechanism \cite{vaswani2023attentionneed} projects these features into Queries ($Q = W_Q \phi(z_t)$), Keys ($K = W_K \tau(c)$), and Values ($V = W_V \tau(c)$). The cross-attention map $\mathbf{A}$ is computed as:
\begin{equation}
    \mathbf{A} = \text{Softmax}\left(\frac{QK^T}{\sqrt{d}}\right)
\end{equation}
where $d$ is the latent projection dimension. Each column in $\mathbf{A}$ represents the spatial layout distribution of a text token \cite{hertz2022prompttopromptimageeditingcross} . In our method, we focus on the test-time optimization of the initialization step. We leverage these native cross-attention maps $\mathbf{A}_l$ as an interpretable structural prior to refine $z_T$ before the reverse denoising process begins, thereby maximizing text-image semantic alignment.

\subsection{Classifier-Free Guidance \cite{ho2022classifierfreediffusionguidance} }
\label{subsec:cfg}

To enhance text conditioning and sample quality, LDMs typically employ CFG during the reverse sampling phase. Instead of relying solely on the conditional prediction $\epsilon_\theta(z_t, t, c)$, it linearly extrapolates between an unconditional prediction $\epsilon_\theta(z_t, t, \emptyset)$ and the conditional one. The modified noise prediction $\tilde{\epsilon}_\theta$ is formulated as:
\begin{equation}
    \tilde{\epsilon}_\theta(z_t, t, c) = \epsilon_\theta(z_t, t, \emptyset) + s \cdot \left( \epsilon_\theta(z_t, t, c) - \epsilon_\theta(z_t, t, \emptyset) \right)
\end{equation}
where $s \geq 1$ is the guidance scale that controls the trade-off between prompt adherence and image diversity.

\section{Methodology}
\label{sec:methodology}

To overcome semantic misallocation and geometric degradation inherent in Euclidean TTO approaches \cite{ma2025inferencetimescalingdiffusionmodels, 10657632, qi2024noisescreatedequallydiffusionnoise, eyring2024renoenhancingonesteptexttoimage} , we propose \textbf{Oracle Noise Optimization} (Algorithm \ref{alg:two_stage_optimization}). Our self-contained framework decomposes the initialization problem into two stages: isolating structural entities to efficiently route optimization energy (Section \ref{subsec:token_weighting}), and executing a strict, prior-preserving gradient ascent on the hyperspherical manifold to eliminate distribution shifts (Section \ref{subsec:spherical_optimization}).

\subsection{Multi-Encoder Token Weighting}
\label{subsec:token_weighting}

A natural language prompt $c$ contains a hierarchical semantic structure, where core entities dictate the visual layout while functional words (e.g., ``a'', ``the'') offer minimal generative value. Treating all tokens equally during noise optimization dilutes the structural conditioning. To mitigate this, we introduce a zero-shot token weighting mechanism that leverages the pre-trained text encoders without relying on external syntactic parsers.

Our core insight is that the intrinsic importance of a token can be quantified by measuring the \textit{representational collapse} of the global sentence embedding when that token is masked. Given a set of pre-trained encoders $\{\mathcal{E}_k\}_{k=1}^K$ (e.g., CLIP-L and CLIP-G in Stable Diffusion XL \cite{podell2023sdxlimprovinglatentdiffusion}), we first extract the base embeddings $E_k = \mathcal{E}_k(c)$. For each valid non-special token $j \in \mathcal{V}$, we construct a lesioned prompt $c^{\setminus j}$. To isolate semantic impact without disrupting the absolute positional encodings of subsequent tokens, this masking operation is implemented by replacing the $j$-th token with a neutral \texttt{[PAD]} token rather than physically deleting it from the sequence. The impact score of token $j$ is then defined as the averaged cosine distance between the base and lesioned embeddings across all encoders:
\begin{equation}
    \mathbf{I}[j] = \frac{1}{K} \sum_{k=1}^{K} \Big( 1 - \cos\big(E_k, E_k^{\setminus j}\big) \Big)
\end{equation}
A high impact score $\mathbf{I}[j]$ indicates that the removal of token $j$ significantly alters the semantic manifold of the prompt, thereby identifying it as a core structural entity. To stabilize the subsequent optimization, we apply an affine mapping function $\Phi$ to normalize these continuous scores into a bounded interval $[w_{\min}, w_{\max}]$. Masked by the valid token indicator $\mathbb{I}_{\mathcal{V}}$, we obtain the dense token weighting vector $\mathbf{M} \in \mathbb{R}^{|c|}$:
\begin{equation}
    \mathbf{M} = \Phi_{[w_{\min}, w_{\max}]}(\mathbf{I}) \odot \mathbb{I}_{\mathcal{V}}
\end{equation}

This weight vector $\mathbf{M}$ serves as a semantic routing map, ensuring that the optimization energy is concentrated exclusively on tokens that genuinely drive the generative layout.

\subsection{Prior-Preserving Spherical Optimization}
\label{subsec:spherical_optimization}

With the semantic routing map $\mathbf{M}$ established, the subsequent and most critical challenge is to optimize the initial latent $z_T$ without destroying the diffusion model's native standard Gaussian prior. Existing methods blindly apply Euclidean updates ($z_T \leftarrow z_T + \eta g$), which is theoretically disastrous. To rigorously justify the necessity of our spherical approach, we must first mathematically formalize the exact topological properties of the latent diffusion prior in the asymptotic limit of high dimensionality.

\subsubsection{The Hyperspherical Geometry of High-Dimensional Noise}

In modern LDMs, the initial noise vector $z$ is a continuous random variable sampled from a standard multivariate Gaussian distribution, i.e., $z \sim \mathcal{N}(\mathbf{0}, \mathbf{I}_D)$. The dimensionality $D$ is extraordinarily high (e.g., $D = 4 \times 64 \times 64 = 16,384$ for Stable Diffusion \cite{rombach2022highresolutionimagesynthesislatent}).

A fundamental property of high-dimensional probability is the concentration of measure. For a standard Gaussian random vector $z \sim \mathcal{N}(\mathbf{0}, \mathbf{I}_D)$, its squared Euclidean norm strictly follows a Chi-squared distribution with $D$ degrees of freedom ($\|z\|^2 \sim \chi^2(D)$), yielding an expected squared norm of $\mathbb{E}[\|z\|^2] = D$. As $D \to \infty$, this leads to the well-known \textit{Gaussian Annulus Theorem}:

\begin{theorem}[Gaussian Annulus Theorem \cite{Vershynin_2018}]
\label{thm:annulus}
As the dimensionality $D \to \infty$, the probability mass of the standard Gaussian distribution $\mathcal{N}(\mathbf{0}, \mathbf{I}_D)$ concentrates entirely within a razor-thin annulus around a Riemannian hypersphere of radius $\sqrt{D}$. Formally, for any arbitrarily small constant $\epsilon \in (0, 1)$:
\begin{equation}
    \lim_{D \to \infty} \mathbb{P}\left( \left| \frac{\|z\|}{\sqrt{D}} - 1 \right| \ge \epsilon \right) = 0
\end{equation}
\end{theorem}

Theorem~\ref{thm:annulus} establishes a profound geometric reality: in high-dimensional spaces ($D \sim 10^4$), the Typical Set (where almost all probability mass resides) is topologically isomorphic to a Riemannian hypersphere $\mathbb{S}^{D-1}(\sqrt{D})$. Any latent variable must strictly reside on this hyperspherical manifold to be considered an in-distribution sample of the diffusion prior.

This absolute geometric constraint immediately exposes the fatal theoretical flaw of traditional Euclidean test-time optimization.


\begin{theorem}
\label{thm:euclidean_degradation}
Let $z \in \mathbb{S}^{D-1}(\sqrt{D})$ be an initialized latent vector, and $\mathcal{L}$ be a non-trivial semantic objective with gradient $g = \nabla_z \mathcal{L} \neq \mathbf{0}$. For an unconstrained Euclidean step $z^{(new)} = z + \eta g$ ($\eta > 0$), the expected latent norm strictly inflates, i.e., $\mathbb{E}[\|z^{(new)}\|^2] > D$, guaranteeing divergence from the Gaussian Typical Set.
\end{theorem}

\begin{proof}
The $\ell_2$-norm of the updated latent vector expands as:
\begin{equation}
    \|z^{(new)}\|^2 = \|z\|^2 + 2\eta \langle z, g \rangle + \eta^2 \|g\|^2
\end{equation}
Preserving the prior norm ($\|z^{(new)}\|^2 = \|z\|^2$) imposes a strict geometric constraint on the gradient:
\begin{equation}
    \langle z, g \rangle = -\frac{\eta}{2} \|g\|^2 < 0
\end{equation}
For a highly non-linear neural objective $\mathcal{L}$, the probability of the gradient $g$ naturally satisfying this exact negative radial projection is strictly measure zero.

\begin{figure}[t]
\centering

\begin{minipage}[t]{0.48\textwidth}
    \vspace{0pt}
    \centering

    \includegraphics[width=\linewidth]{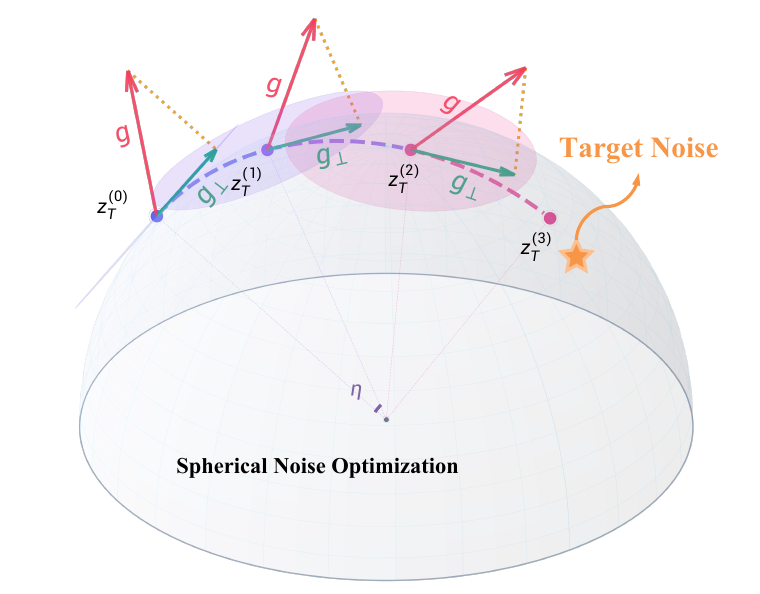}
    
    \vspace{0.3em}
    {\small \textbf{(a)} Prior-preserving spherical optimization.}

    \vspace{0.8em}

    \includegraphics[width=\linewidth]{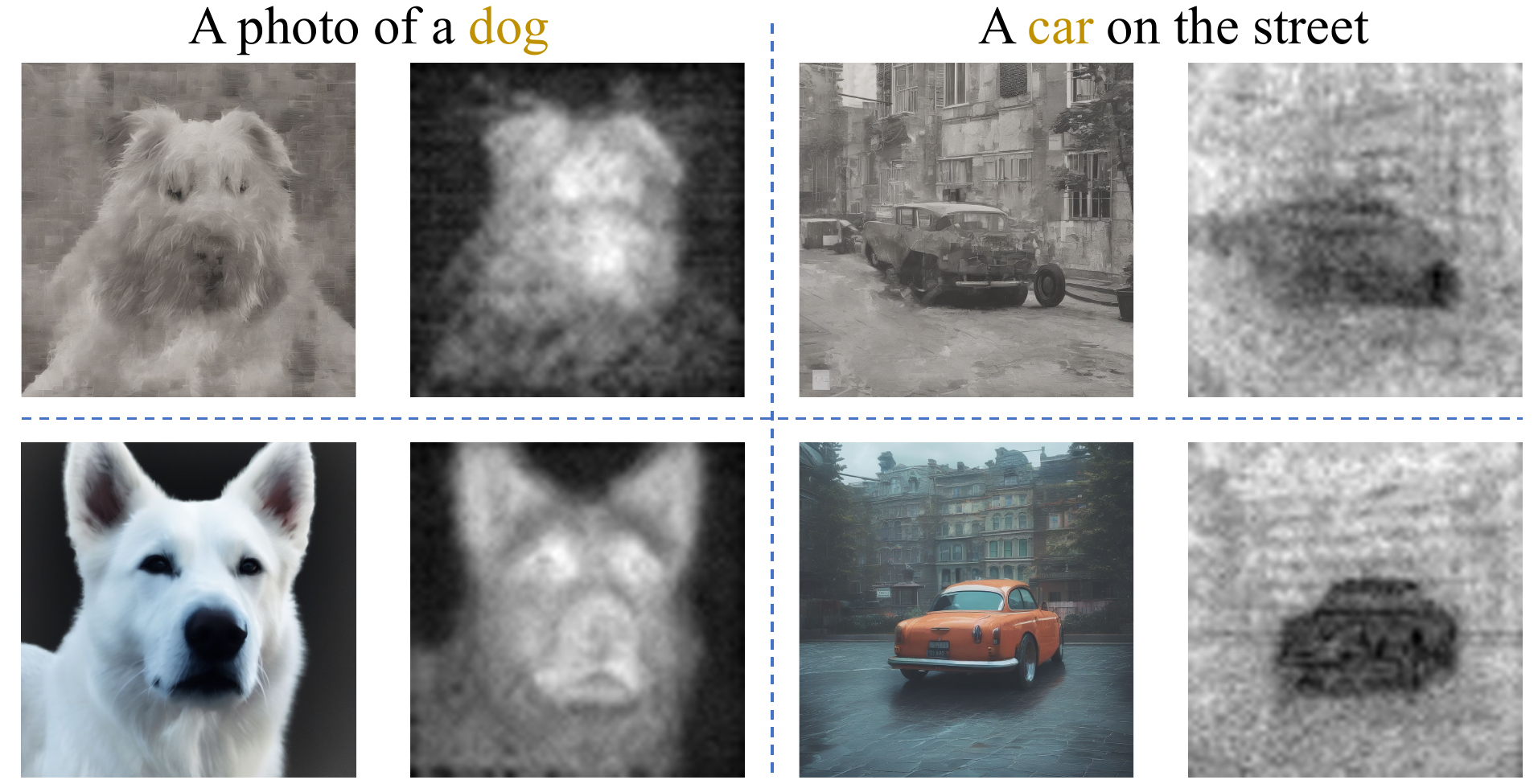}
    
    \vspace{0.3em}
    {\small \textbf{(b)} Cross-attention visualization.}
\end{minipage}
\hfill
\begin{minipage}[t]{0.48\textwidth}
    \vspace{0pt}
    \captionsetup{type=algorithm}
    \caption{Oracle Noise Optimization}
    \label{alg:two_stage_optimization}

    \footnotesize
    \begin{algorithmic}[1]
    \Require Diffusion Model $\mathcal{M}_\theta$, Encoders $\{\mathcal{E}_k\}_{k=1}^K$, prompt $c$, latent $z_T$, step size $\eta$, iterations $N$, bounds $[w_{\min}, w_{\max}]$, weights $\{\alpha_l\}$, guidance scale $s$, max timestep $T$
    \Ensure Aligned latent $z_T$

    \Statex \rule{\linewidth}{0.4pt}
    \Statex \textbf{Step 1: Multi-Encoder Token Weighting}

    \State $\mathcal{V} \leftarrow \{j \mid c_j \notin \mathcal{T}_{\text{special}} \}; \quad E_k \leftarrow \mathcal{E}_k(c), \forall k \in \{1, \dots, K\}$
    \For{$j \in \mathcal{V}$}
        \State $c^{\setminus j} \leftarrow \text{Replace}(c, c_j, \texttt{[PAD]}); \quad E_k^{\setminus j} \leftarrow \mathcal{E}_k(c^{\setminus j}), \forall k$
        \State $\mathbf{I}[j] \leftarrow \frac{1}{K} \sum_{k=1}^{K} \big( 1 - \cos(E_k, E_k^{\setminus j}) \big)$
    \EndFor
    \State $\mathbf{M} \leftarrow \Phi_{[w_{\min}, w_{\max}]}(\mathbf{I}) \odot \mathbb{I}_{\mathcal{V}}$

    \Statex \textbf{Step 2: Spherical Noise Optimization}

    \State $c_e \leftarrow \text{Embed}(c); \quad c_\emptyset \leftarrow \text{Embed}(\emptyset)$
    \For{$i = 1, \dots, N$}
        \State $\{Q_l^c, K_l^c\} \leftarrow \mathcal{M}_\theta(z_T, T, c_e); \quad \{Q_l^\emptyset, K_l^\emptyset\} \leftarrow \mathcal{M}_\theta(z_T, T, c_\emptyset)$
        \State $\mathcal{L} \leftarrow 0$
        \For{layer $l$}
            \State $L_l^c \leftarrow \frac{Q_l^c (K_l^c)^T}{\sqrt{d}}; \quad L_l^\emptyset \leftarrow \frac{Q_l^\emptyset (K_l^\emptyset)^T}{\sqrt{d}}$
            \State $\tilde{\mathbf{A}}_l \leftarrow \text{Softmax}\big(L_l^\emptyset + s \cdot (L_l^c - L_l^\emptyset)\big)$
            \State $\mathcal{L} \leftarrow \mathcal{L} + \alpha_l \sum_{p, j} \left( \tilde{\mathbf{A}}_{l,p,j} \cdot \mathbf{M}_j \right)$
        \EndFor
        \State $g \leftarrow \nabla_{z_T} \mathcal{L}; \quad g_{\perp} \leftarrow g - \frac{\langle z_T, g \rangle}{\|z_T\|^2} z_T$
        \State $z_T \leftarrow z_T \cos \eta + \|z_T\| \frac{g_{\perp}}{\|g_{\perp}\|} \sin \eta$
    \EndFor
    \State \Return $z_T$
    \end{algorithmic}
\end{minipage}

\caption{\textbf{Visualization of Oracle Noise.}
\textbf{Left-top:} Prior-preserving spherical optimization on the hypersphere.
\textbf{Left-bottom:} Cross-attention visualization showing the interaction between entity tokens and latent variables.
\textbf{Right:} Pseudocode of the Oracle Noise optimization pipeline.}
\label{fig:oracle_overview}
\end{figure}

    Crucially, modern diffusion backbones heavily rely on structural normalization
    layers (e.g., GroupNorm) prior to cross-attention projections. This renders the
    pre-softmax attention logits, and consequently our objective $\mathcal{L}(z)$,
    largely scale-invariant with respect to the norm of the latent input.
    Mathematically, $\mathcal{L}$ acts approximately as a degree-zero homogeneous
    function: $\mathcal{L}(c z) \approx \mathcal{L}(z)$ for any scalar $c > 0$.
    By Euler's Homogeneous Function Theorem, the gradient of a strictly
    scale-invariant function is identically orthogonal to its input vector,
    establishing $\langle z, g \rangle \approx 0$.

    Taking the expectation of the expanded norm gives:
    \begin{equation}
    \begin{aligned}
        \mathbb{E}\!\left[\|z^{(new)}\|^2\right]
        &\approx D + \eta^2 \mathbb{E}\!\left[\|g\|^2\right].
    \end{aligned}
    \end{equation}

    Since $\eta > 0$ and $g \neq \mathbf{0}$, the strictly positive quadratic term
    ensures $\eta^2 \mathbb{E}[\|g\|^2] > 0$, consequently yielding:
    \begin{equation}
        \mathbb{E}\!\left[\|z^{(new)}\|^2\right] > D .
    \end{equation}

Therefore, the updated latent $z^{(new)}$ monotonically escapes the native
$\sqrt{D}$-hypersphere defined in Theorem~\ref{thm:annulus}.
\end{proof}

Theorem~\ref{thm:euclidean_degradation} proves that an unconstrained Euclidean step systematically forces the latent out of the prior's Typical Set, inevitably triggering the severe visual degradation and color over-saturation observed in baseline methods. While fixing the norm technically yields a uniform spherical distribution, it remains asymptotically equivalent to the high-dimensional standard Gaussian \cite{Vershynin_2018}. This mathematical reality necessitates a fundamental paradigm shift to \textbf{Spherical CFG-Aware Optimization}.

\subsubsection{Spherical CFG-Aware Optimization}
To optimize $z_T$ while strictly remaining on this $\sqrt{D}$ hypersphere, we formulate a \textit{CFG-aware} objective. Standard approaches naively optimize the conditional attention maps, ignoring the Classifier-Free Guidance applied during reverse sampling, causing the optimization trajectory to diverge from the final generation path. Instead, we utilize the frozen denoising network $\mathcal{M}_\theta$ at the maximum noise timestep $T$ to extract pre-softmax attention logits for both conditional ($L^c$) and unconditional

($L^\emptyset$) forward passes. We perform CFG extrapolation strictly within the logit space:
\begin{equation}
    \tilde{\mathbf{A}}_l = \text{Softmax}\left( L^\emptyset_l + s \cdot (L^c_l - L^\emptyset_l) \right)
\end{equation}

where $s$ is the guidance scale. We compute our objective function $\mathcal{L}(z_T)$ as the weighted sum of these CFG-aware maps across all spatial positions $p$ and token indices $j$, modulated by our token weighting vector $\mathbf{M}$ and layer-wise scalars $\alpha_l$:
\begin{equation}
    \mathcal{L}(z_T) = \sum_{l} \alpha_l \sum_{p} \sum_{j} \left( \tilde{\mathbf{A}}_{l, p, j} \cdot \mathbf{M}_j \right)
\end{equation}

Instead of treating all attention layers equally, we assign increasing weights (e.g., $\alpha_l \in \{1.0, 1.5, 2.0\}$) to the shallow, middle, and deep layers, respectively. This design is empirically motivated by interpretability studies on diffusion models \cite{hertz2022prompttopromptimageeditingcross, tang2022daaminterpretingstablediffusion, tumanyan2022plugandplaydiffusionfeaturestextdriven}. Prior arts demonstrate that while early encoding layers capture low-level global context, the later decoding layers (or deeper transformer blocks) dominate fine-grained semantic alignment and spatial layout generation. Therefore, applying stronger constraints on these deep layers yields more accurate text-to-image semantic matching.

To execute the update without escaping the hypersphere, we first compute the raw Euclidean gradient $g = \nabla_{z_T} \mathcal{L}$. We then orthogonally project this gradient onto the tangent plane of the hypersphere at the current latent point $z_T$:
\begin{equation}
    g_{\perp} = g - \left( \frac{\langle z_T, g \rangle}{\|z_T\|^2} \right) z_T
\end{equation}

Finally, we perform a \textit{spherical geodesic step} strictly along the great circle defined by $z_T$ and the normalized tangent direction $g_{\perp}/\|g_{\perp}\|$. Parameterized by the angular step size $\eta$, the precise update rule is formulated as:
\begin{equation}
    z_T \leftarrow z_T \cos \eta + \|z_T\| \frac{g_{\perp}}{\|g_{\perp}\|} \sin \eta
    \label{eq:geodesic_step}
\end{equation}

Equation (\ref{eq:geodesic_step}) provides a strict mathematical guarantee: the $\ell_2$-norm of the latent remains exactly invariant ($\|z_T^*\| \equiv \|z_T\|$) throughout the $N$ optimization iterations. The final optimized latent $z_T^*$ safely navigates the highly non-linear attention landscape to embed critical spatial layouts, whilst flawlessly preserving its identity as a valid sample drawn from $\mathcal{N}(\mathbf{0}, \mathbf{I})$.

Because our geodesic update rigorously prevents norm inflation, we completely bypass the geometric degradation that plagues previous works. More importantly, this absolute structural stability permits the use of significantly larger, more aggressive optimization step sizes ($\eta$), drastically accelerating convergence without risking divergence or triggering reward hacking vulnerabilities.

\section{Empirical Analysis}

\subsection{Experimental Settings}

\paragraph{Datasets \& Metrics.}
We evaluate our framework on four diverse benchmarks to capture different generative capabilities. We use \textbf{MS-COCO 2017 \cite{lin2015microsoftcococommonobjects}} (5k val) to test zero-shot fidelity, \textbf{DrawBench \cite{saharia2022photorealistic}} for complex spatial relations, and \textbf{GenEval \cite{ghosh2023genevalobjectfocusedframeworkevaluating}} for fine-grained compositional reasoning and object counting. Additionally, we utilize \textbf{Pick-a-Pic \cite{kirstain2023pickapicopendatasetuser}} to assess alignment with complex human preferences. To comprehensively quantify performance, generative quality and diversity are measured using \textbf{FID \cite{Seitzer2020FID}} and \textbf{Vendi Score \cite{friedman2023vendiscorediversityevaluation}}, while \textbf{CLIP Score \cite{radford2021learningtransferablevisualmodels}} evaluates overall text-image semantic matching. Finally, detailed intent alignment and visual appeal are rigorously assessed via \textbf{HPSv2 \cite{wu2023human}}, \textbf{ImageReward \cite{xu2023imagerewardlearningevaluatinghuman}}, \textbf{PickScore}, and \textbf{Aesthetics \cite{schuhmann2022laion5bopenlargescaledataset}}.

\begin{table}[H]
  \centering
  \caption{Quantitative comparison of optimization methods on SDXL and SD3.5 Medium under conditional generation}
  \label{tab:main_results}
  \resizebox{\textwidth}{!}{ 
  \begin{tabular}{lllccccccc}
    \toprule
    \textbf{Model} & \textbf{Dataset} & \textbf{Method} & \textbf{HPSv2 ($\uparrow$)} & \textbf{ImageReward ($\uparrow$)} & \textbf{PickScore ($\uparrow$)} & \textbf{Aesthetics ($\uparrow$)} & \textbf{Vendi Score ($\uparrow$)} & \textbf{CLIP Score ($\uparrow$)} & \textbf{Time (s) $\downarrow$} \\
    \midrule
    
    \multirow{8}{*}{SDXL} & \multirow{4}{*}{Pick-a-Pic} 
    & Gaussian & 24.53 & -1.01 & 17.55 & 5.97 & 8.79 & 57.05 & -- \\
    & & InitNO & 25.20 & -0.60 & 17.60 & 6.05 & 9.50 & 58.80 & 35.0 \\
    & & Stable Noise & 25.40 & -0.45 & 17.62 & \textbf{6.14} & 10.20 & 58.20 & $>600$ \\
    & & \textbf{Oracle Noise (Ours)} & \textbf{25.85} & \textbf{-0.21} & \textbf{17.72} & 6.13 & \textbf{11.33} & \textbf{59.38} & \textbf{2.0} \\
    \cmidrule(lr){2-10}
    & \multirow{4}{*}{DrawBench} 
    & Gaussian & 24.67 & -1.16 & 19.26 & 5.56 & 9.76 & 49.68 & -- \\
    & & InitNO & 25.10 & -0.90 & 19.30 & 5.55 & 10.50 & 51.80 & 35.0 \\
    & & Stable Noise & 25.35 & -0.85 & 19.32 & \textbf{5.59} & 11.20 & 52.00 & $>600$ \\
    & & \textbf{Oracle Noise (Ours)} & \textbf{25.69} & \textbf{-0.68} & \textbf{19.44} & 5.48 & \textbf{13.02} & \textbf{54.23} & \textbf{2.0} \\
    
    \midrule
    
    \multirow{8}{*}{SD3.5-M} & \multirow{4}{*}{Pick-a-Pic} 
    & Gaussian & 25.21 & -0.61 & 17.67 & \textbf{5.67} & 11.99 & 57.97 & -- \\
    & & InitNO & 25.35 & -0.45 & 17.68 & 5.65 & 12.30 & 59.50 & 35.0 \\
    & & Stable Noise & 25.45 & -0.40 & 17.68 & \textbf{5.67} & 12.50 & 59.80 & $>600$ \\
    & & \textbf{Oracle Noise (Ours)} & \textbf{25.76} & \textbf{-0.22} & \textbf{17.72} & 5.63 & \textbf{13.49} & \textbf{61.49} & \textbf{2.0} \\
    \cmidrule(lr){2-10}
    & \multirow{4}{*}{DrawBench} 
    & Gaussian & 25.76 & -0.55 & 19.39 & 5.19 & 12.50 & 58.76 & -- \\
    & & InitNO & 25.80 & -0.50 & 19.39 & 5.20 & 13.00 & 58.80 & 35.0 \\
    & & Stable Noise & 25.82 & -0.48 & 19.39 & \textbf{5.23} & 13.20 & 58.85 & $>600$ \\
    & & \textbf{Oracle Noise (Ours)} & \textbf{25.97} & \textbf{-0.40} & \textbf{19.42} & 5.21 & \textbf{14.60} & \textbf{59.18} & \textbf{2.0} \\
    \bottomrule
  \end{tabular}
  }
\end{table}

\begin{table}[H]
  \centering
  \caption{Quantitative comparison of optimization methods on SDXL and SD3.5 Medium under CFG generation}
  \label{tab:main_results_cfg}
  \resizebox{\textwidth}{!}{ 
  \begin{tabular}{lllccccccc}
    \toprule
    \textbf{Model} & \textbf{Dataset} & \textbf{Method} & \textbf{HPSv2 ($\uparrow$)} & \textbf{ImageReward ($\uparrow$)} & \textbf{PickScore ($\uparrow$)} & \textbf{Aesthetics ($\uparrow$)} & \textbf{Vendi Score ($\uparrow$)} & \textbf{CLIP Score ($\uparrow$)} & \textbf{Time (s) $\downarrow$} \\
    \midrule
    
    \multirow{8}{*}{SDXL} & \multirow{4}{*}{Pick-a-Pic} 
    & Gaussian & 27.88 & 0.58 & 17.73 & 6.03 & 14.86 & 71.83 & -- \\
    & & InitNO & 27.91 & 0.60 & 17.76 & 6.05 & \textbf{15.50} & 72.10 & 18.9 \\
    & & Stable Noise & 27.92 & 0.61 & 17.78 & 6.10 & 15.00 & 72.20 & 35.0 \\
    & & \textbf{Oracle Noise (Ours)} & \textbf{27.97} & \textbf{0.67} & \textbf{17.85} & \textbf{6.13} & 15.43 & \textbf{73.00} & \textbf{2.0} \\
    \cmidrule(lr){2-10}
    & \multirow{4}{*}{DrawBench} 
    & Gaussian & 28.07 & 0.51 & 19.91 & 5.43 & 16.97 & 69.79 & -- \\
    & & InitNO & 28.12 & 0.53 & 19.94 & 5.45 & 17.50 & 69.90 & 18.9 \\
    & & Stable Noise & 28.15 & 0.53 & \textbf{20.05} & 5.46 & 17.80 & 69.95 & 35.0 \\
    & & \textbf{Oracle Noise (Ours)} & \textbf{28.25} & \textbf{0.57} & 20.02 & \textbf{5.48} & \textbf{18.89} & \textbf{70.21} & \textbf{2.0} \\
    
    \midrule
    
    \multirow{8}{*}{SD3.5-M} & \multirow{4}{*}{Pick-a-Pic} 
    & Gaussian & 28.16 & 0.81 & 17.79 & 5.84 & 15.10 & 69.12 & -- \\
    & & InitNO & 28.35 & 0.84 & 17.81 & 5.87 & 15.50 & 69.45 & 18.0 \\
    & & Stable Noise & 28.45 & 0.85 & 17.82 & \textbf{5.95} & 15.60 & 69.60 & 32.0 \\
    & & \textbf{Oracle Noise (Ours)} & \textbf{28.76} & \textbf{0.93} & \textbf{17.85} & 5.93 & \textbf{15.88} & \textbf{70.14} & \textbf{2.0} \\
    \cmidrule(lr){2-10}
    & \multirow{4}{*}{DrawBench} 
    & Gaussian & 28.98 & 0.79 & 18.44 & 5.35 & 18.37 & 71.37 & -- \\
    & & InitNO & 29.02 & 0.82 & \textbf{18.50} & 5.36 & 18.45 & 71.50 & 18.0 \\
    & & Stable Noise & 29.05 & 0.84 & 18.44 & 5.36 & 18.50 & 71.60 & 32.0 \\
    & & \textbf{Oracle Noise (Ours)} & \textbf{29.13} & \textbf{0.89} & 18.45 & \textbf{5.40} & \textbf{18.72} & \textbf{72.10} & \textbf{2.0} \\
    \bottomrule
  \end{tabular}
  }
\end{table}

\paragraph{Models \& Parameters.}
To demonstrate architectural generalization, we implement our method across multiple diffusion backbones. All inference timings and experiments were conducted on a single NVIDIA RTX 4090 GPU. For \textbf{SDXL \cite{podell2023sdxlimprovinglatentdiffusion}}, we use the standard 50-step setup (CFG 7.5), while \textbf{SDXL-Turbo \cite{sauer2023adversarialdiffusiondistillation}} is tested in an extreme 1-step regime. Both utilize either single-step ($\eta = 0.05$) or multi-step ($\eta = 0.005, N = 10$) optimization. We also evaluate the newer \textbf{SD3.5-Medium \cite{esser2024scalingrectifiedflowtransformers}} (28 steps, CFG 4.5) with single-step ($\eta = 0.01$) or multi-step ($\eta = 0.005, N = 2$) configurations. To maximize semantic routing efficiency during attention optimization, we apply hierarchical layer weights $\alpha_l = \{1.0, 1.5, 2.0\}$ to progressively constrain semantically dense layers. Specifically, these correspond to the first down-block, mid-block, and first up-block in SDXL models, and transformer blocks 0, 11, and 23 in SD3.5-Medium. For the token weighting mechanism, the affine mapping interval bounds $[w_{\min}, w_{\max}]$ are empirically set to $[0.5, 3.0]$.

\subsection{Main Experiments}


\textbf{State-of-the-Art Alignment within a 2-Second Budget.} We first evaluate the core text-image alignment and aesthetic quality under standard conditional generation (Table~\ref{tab:main_results}). Across both SDXL and SD3.5-Medium architectures, Oracle Noise consistently establishes new state-of-the-art performance, achieving the highest HPSv2, ImageReward, and CLIP scores on the complex Pick-a-Pic and DrawBench datasets. Crucially, existing test-time optimization baselines (e.g., InitNO, Stable Noise) incur prohibitive inference latency, often ranging from 35 seconds to over 10 minutes per image due to their inefficient Euclidean search. In stark contrast, by mathematically eliminating norm inflation and safely utilizing aggressive step sizes, Oracle Noise achieves superior alignment and sample diversity (Vendi Score) within a strict 2-second budget. This extraordinary efficiency robustly generalizes even to extreme few-step distilled models like SDXL-Turbo (Table~\ref{tab:distilled_model_results}), where our method significantly outperforms ReNO across all human preference metrics while accelerating the optimization overhead by $15\times$.
\begin{table}[t]
  \centering
  \caption{Comparison of optimization methods on SDXL-Turbo. While ReNO requires dozens of steps, Oracle Noise achieves superior quality within a strict 2-second budget.}
  \label{tab:distilled_model_results}
  \setlength{\tabcolsep}{2.5pt}
  \footnotesize
  \begin{tabular}{llccccccc}
    \toprule
    \textbf{Dataset} & \textbf{Method} & \textbf{HPSv2} & \textbf{ImageReward} & \textbf{PickScore} & \textbf{Aesthetics} & \textbf{Vendi Score} & \textbf{CLIP Score} & \textbf{Time} \\
    \midrule
    \multirow{3}{*}{Pick-a-Pic} 
    & Gaussian & 27.93 & 0.72 & 17.69 & 5.98 & 14.41 & 70.26 & -- \\
    & ReNO     & 28.01 & 0.75 & 17.75 & 6.01 & 14.80 & 70.50 & 30s \\
    & \textbf{Oracle} & \textbf{28.10} & \textbf{0.78} & \textbf{17.83} & \textbf{6.05} & \textbf{15.34} & \textbf{70.93} & \textbf{2s} \\
    \midrule
    \multirow{3}{*}{DrawBench} 
    & Gaussian & 28.41 & 0.67 & \textbf{19.99} & 5.49 & 17.11 & 69.25 & -- \\
    & ReNO     & 28.55 & 0.71 & 19.95 & 5.55 & 17.40 & 69.40 & 30s \\
    & \textbf{Oracle} & \textbf{28.71} & \textbf{0.76} & 19.98 & \textbf{5.64} & \textbf{17.72} & \textbf{69.57} & \textbf{2s} \\
    \bottomrule
  \end{tabular}
\end{table}

\begin{table}[t]
\centering
\caption{\textbf{GenEval results on SDXL.} Oracle Noise improves compositional reasoning across categories.}
\label{tab:geneval_single}
\footnotesize
\setlength{\tabcolsep}{2.5pt}
\renewcommand{\arraystretch}{1.08}
\begin{tabular}{lccccccc}
\toprule
\textbf{Method} & \textbf{Single Object} $\uparrow$ & \textbf{Two Objects} $\uparrow$ & \textbf{Counting} $\uparrow$ & \textbf{Color} $\uparrow$ & \textbf{Position} $\uparrow$ & \textbf{Color Attribution} $\uparrow$ & \textbf{Overall} $\uparrow$ \\
\midrule
Gaussian & 97.60 & 70.50 & 34.50 & 86.80 & 10.20 & 20.10 & 53.18 \\
\textbf{Oracle} & \textbf{99.75} & \textbf{74.25} & \textbf{45.00} & \textbf{87.18} & \textbf{10.50} & \textbf{23.40} & \textbf{56.59} \\
\bottomrule
\end{tabular}
\end{table}

\begin{figure*}[t]
\centering

\begin{minipage}[t]{0.58\textwidth}
\vspace{0pt}

\captionof{table}{Conditional generation on COCO val2017 (5k). We compare standard Gaussian noise with the proposed Oracle Noise initialization on SDXL and SD 3.5-M, reporting both FID and CLIP scores to evaluate image fidelity and text-image alignment.}
\label{tab:coco_results}
\centering
\small
\resizebox{\linewidth}{!}{
\begin{tabular}{lcccc}
    \toprule
    \multirow{2}{*}{\textbf{Initialization}} & \multicolumn{2}{c}{\textbf{SDXL}} & \multicolumn{2}{c}{\textbf{SD 3.5-M}} \\
    \cmidrule(lr){2-3} \cmidrule(lr){4-5}
    & \textbf{FID} $\downarrow$ & \textbf{CLIP} $\uparrow$ & \textbf{FID} $\downarrow$ & \textbf{CLIP} $\uparrow$ \\
    \midrule
    Gaussian Noise & 52.80 & 0.2157 & 28.94 & 0.2379 \\
    \textbf{Oracle Noise (Ours)} & \textbf{38.96} & \textbf{0.2717} & \textbf{25.96} & \textbf{0.2523} \\
    \bottomrule
\end{tabular}
}

\vspace{1.5em}

\captionof{table}{Ablation of Oracle Noise on SDXL. \textbf{IR}: ImageReward, \textbf{Pick}: PickScore, \textbf{Aes.}: Aesthetics. \textbf{w/o}: without. Progressive ablation from the full method confirms the necessity of each component.}
\label{tab:ablation_study}
\centering
\small
\setlength{\tabcolsep}{3.5pt}

\begingroup
\renewcommand{\arraystretch}{1.32}
\resizebox{\linewidth}{!}{
\begin{tabular}{lcccccc}
    \toprule
    \textbf{Method} & \textbf{HPSv2} & \textbf{IR} & \textbf{Pick} & \textbf{Aes} & \textbf{Vendi} & \textbf{CLIP} \\
    \midrule
    \textbf{Oracle (Full Method)} & \textbf{25.85} & \textbf{-0.21} & \textbf{17.72} & \textbf{6.13} & \textbf{11.33} & \textbf{59.38} \\
    \addlinespace[2pt]
    \midrule
    \quad w/o Multi-step Integration & 25.10 & -0.70 & 17.60 & 5.92 & 10.20 & 56.40 \\
    \addlinespace[1.5pt]
    \quad w/o CLIP Token Weighting & 24.85 & -0.82 & 17.56 & 5.80 & 9.85 & 55.20 \\
    \addlinespace[1.5pt]
    \quad w/o Spherical Const. (Eucl.) & 24.61 & -0.90 & 17.53 & 5.72 & 9.49 & 54.11 \\
    \bottomrule
\end{tabular}
}
\endgroup

\end{minipage}
\hfill
\begin{minipage}[t]{0.4\textwidth}
\vspace{0pt}
\centering
\includegraphics[width=\linewidth]{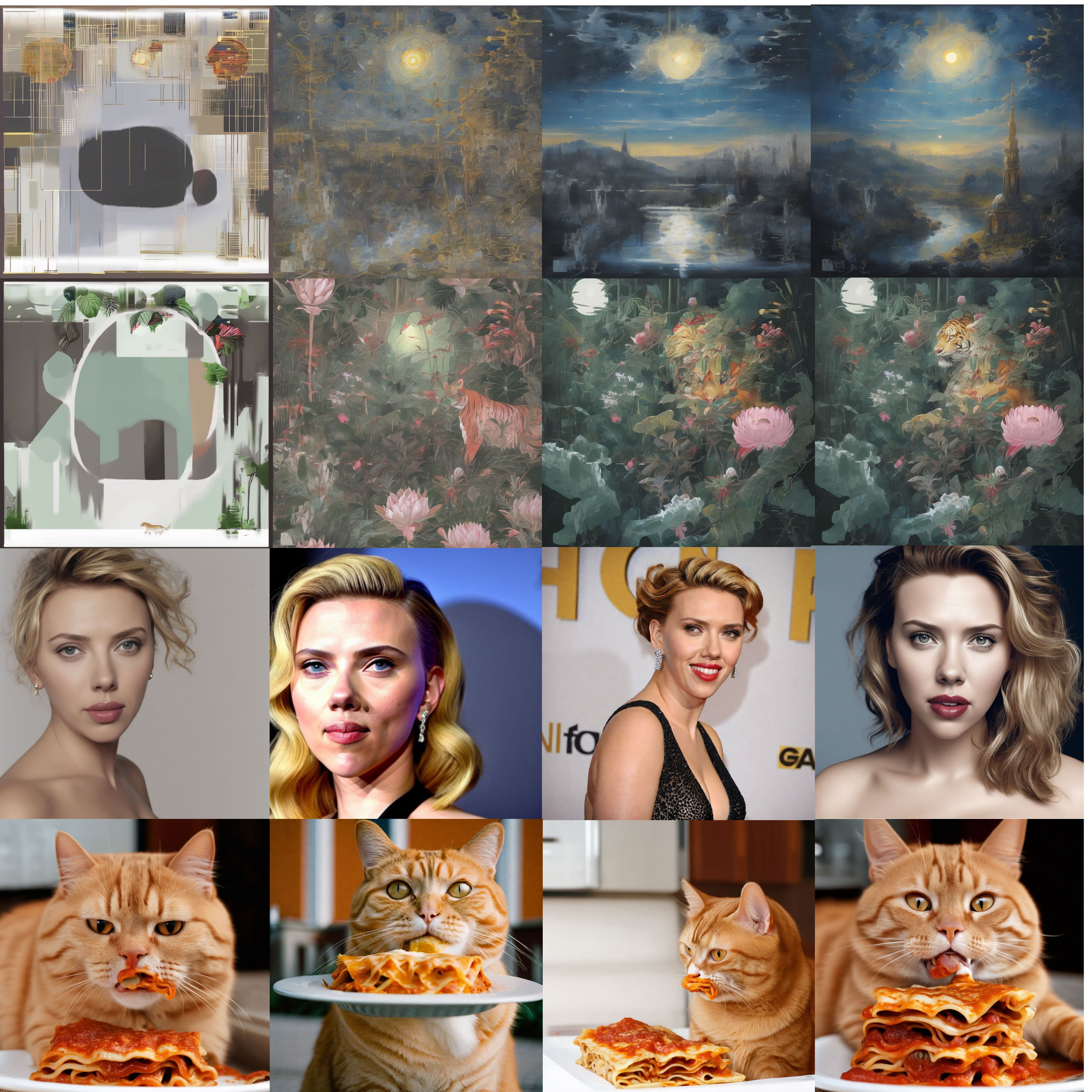}
\vspace{-1em}
\captionof{figure}{Qualitative comparison and ablation of Oracle Noise. Top: ablations from Euclidean single-step to full method. Bottom: CFG comparison with Gaussian noise, Initio, Stable Noise, and Oracle Noise. Oracle Noise improves visual quality and semantic alignment while preserving manifold.}
\label{fig:noise_visualization}
\end{minipage}

\end{figure*}

\begin{figure}[t]
  \centering
  \includegraphics[width=\textwidth]{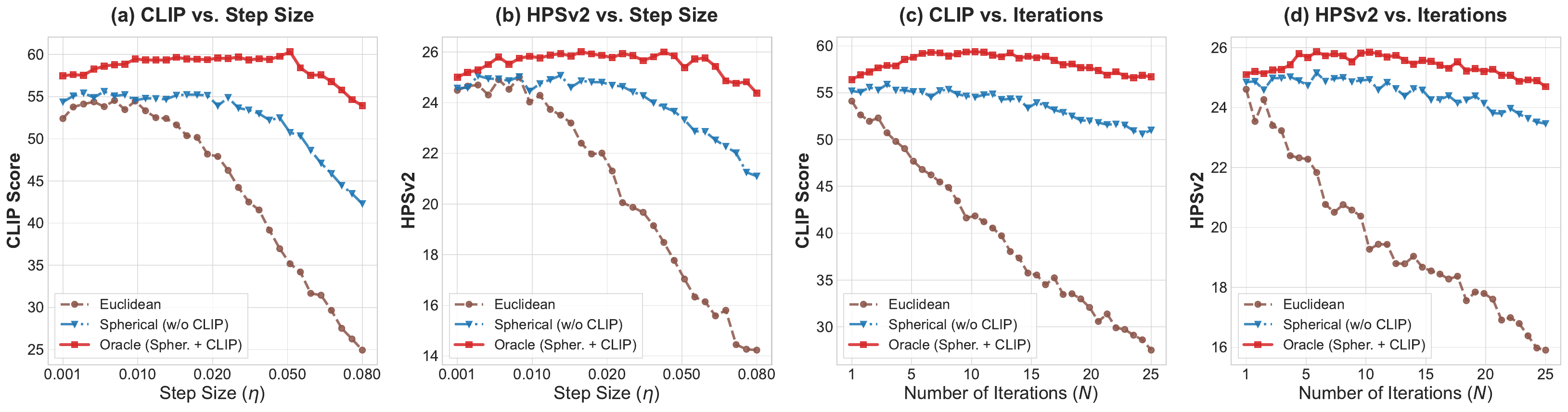}
  \caption{\textbf{Comprehensive ablation across hyperparameter dynamics ($\eta$ and $N$).} We trace the performance of fundamental optimization variants across varying step sizes and iteration counts. \textbf{(a, b):} Unconstrained Euclidean updates inevitably crash under aggressive step sizes due to latent norm inflation. While the pure Spherical formulation improves geometric stability, its optimization ceiling is severely limited by attention over-saturation on functional words. \textbf{(c, d):} By integrating CLIP-based structural routing, our full \textbf{Oracle} framework decisively breaches this ceiling, rapidly converging to peak generative fidelity at $\eta=0.05$ and $N=10$. Notably, the trajectories faithfully capture the ``geodesic overshoot'' phenomenon at extreme bounds ($\eta \ge 0.08$, $N \ge 20$), empirically validating our optimal hyperparameter selection while demonstrating Oracle's superior robustness within the practical inference budget.}
  \label{fig:hyperparams_comprehensive}
\end{figure}

\textbf{Prior Preservation under Practical CFG Generation.} A fundamental flaw of existing Euclidean-based methods is their tendency to destroy the native Gaussian prior, leading to severe visual degradation (e.g., color over-saturation) when deployed with Classifier-Free Guidance (CFG). To rigorously evaluate this, we assess the methods under practical CFG generation (Table~\ref{tab:main_results_cfg}). While baseline methods occasionally manage marginal gains in isolated metrics, their out-of-distribution latents struggle to stably interact with the CFG mechanism, leading to sub-optimal aesthetics. Conversely, because our spherical geodesic update mathematically guarantees the preservation of the $\mathcal{N}(\mathbf{0}, \mathbf{I})$ distribution, Oracle Noise completely mitigates CFG-induced structural collapse. It maintains absolute dominance across ImageReward, PickScore, and CLIP Score, demonstrating that our optimization path remains perfectly aligned with the pre-trained generative manifold.

\textbf{Superior Fine-Grained Compositional Reasoning.} Beyond aesthetic alignment, we challenge our framework with the rigorous GenEval benchmark (Table~\ref{tab:geneval_single}) to assess its capability in fine-grained semantic control. Traditional equal-weighting approaches fail to parse complex multi-object prompts, resulting in semantic bleed and layout confusion. Empowered by our multi-encoder representational collapse mechanism, Oracle Noise dynamically isolates and prioritizes core structural entities. Consequently, it consistently improves compositional reasoning across all categories. Most notably, our method achieves substantial gains in accurate counting (\textbf{Cnt.}, +10.50) and multiple attribute binding (\textbf{Attr.}, +3.30). This proves that directing optimization energy strictly toward high-value structural tokens is the key to unlocking flawless detailed control without relying on external proxy models.

\subsection{Ablation Study}
\label{subsec:ablation_study}

We ablate the core components of our framework in Table~\ref{tab:ablation_study} and analyze its hyperparameter dynamics ($\eta$ and $N$) in Figure~\ref{fig:hyperparams_comprehensive}.

\textbf{Component Ablation.} We validate the contribution of each core module by progressively ablating them from the full pipeline. As shown in Table~\ref{tab:ablation_study}, the complete \textbf{Oracle} method---which integrates multi-step geodesic updates with multi-encoder token weighting---achieves peak generative fidelity and sample diversity across all metrics. Removing the multi-step mechanism and reverting to a single-step spherical update (\textit{w/o Multi-step}) causes a noticeable drop in aesthetic quality (HPSv2: $25.85 \to 25.10$) and text-image alignment (CLIP: $59.38 \to 56.40$), confirming that iterative geodesic navigation is crucial for deeply embedding complex spatial layouts. Further removing the CLIP-based token weighting (\textit{w/o Token Weighting}) dilutes the optimization energy across non-informative functional words. This lack of targeted routing leads to an attention over-saturation bottleneck, visibly degrading intent alignment (CLIP Score: $56.40 \to 55.20$). Finally, replacing our strict spherical constraint with a baseline unconstrained Euclidean update (\textit{w/o Spherical Const.}) exposes the fundamental geometric flaw discussed in Section~\ref{subsec:spherical_optimization}. Driven by latent norm inflation under aggressive step sizes, the optimization trajectory diverges, resulting in the lowest aesthetic alignment (Aes: $5.72$) and geometric stability.

\textbf{Hyperparameter Dynamics.} Our comprehensive analysis (Figure~\ref{fig:hyperparams_comprehensive} c, d) demonstrates how these variants behave over varying iterations. Guided by CLIP-based structural routing, the full Oracle framework decisively breaches the previous optimization ceiling, rapidly converging to peak generative fidelity at $\eta=0.05$ and $N=10$. Notably, the trajectories faithfully capture the ``geodesic overshoot'' phenomenon at extreme bounds ($\eta \ge 0.08, N \ge 20$). This empirically validates our optimal hyperparameter selection while demonstrating Oracle's superior robustness within the practical 2-second inference budget.

\section{Conclusion}
\label{sec:conclusion}

In this paper, we introduced \textbf{Oracle Noise}, a zero-shot, prior-preserving framework that fundamentally resolves the theoretical and empirical bottlenecks of test-time noise optimization in diffusion models. We mathematically demonstrated that traditional Euclidean updates inevitably destroy the high-dimensional Gaussian prior. To circumvent this, we reframed noise initialization as a rigorous Riemannian hyperspherical optimization problem. By synergizing a novel, parser-free token weighting mechanism with a strict spherical geodesic update, Oracle Noise perfectly preserves the native noise distribution while efficiently routing optimization energy to core structural entities. Extensive evaluations confirm that our approach completely mitigates Euclidean-induced visual degradation and reward hacking, achieving state-of-the-art semantic alignment, generation diversity, and aesthetic quality without relying on external black-box models. Ultimately, Oracle Noise unlocks the true structural potential of the initial latent canvas, paving a principled path for highly interpretable generative control.
\appendix

\section{Additional Experimental Results}
\begin{table}[H]
  \centering
  \caption{\textbf{Quantitative comparison of different initial noise distributions on SDXL.} We evaluate Gaussian noise, Golden Noise, and our proposed Oracle Noise across Pick-a-Pic and DrawBench datasets. Results show that Oracle Noise achieves superior performance in aesthetic quality and competitive alignment scores compared to established baselines.}
  \label{tab:golden_vs_oracle_updated}
  \small
  \setlength{\tabcolsep}{3.2pt} 
  \begin{tabular}{clccccc}
    \toprule
    \textbf{Dataset} & \textbf{Method} & \textbf{HPS$\uparrow$} & \textbf{IR$\uparrow$} & \textbf{Pick$\uparrow$} & \textbf{Aes.$\uparrow$} & \textbf{CLIP$\uparrow$} \\
    \midrule
    
    \multirow{3}{*}{Pick-a-Pic} 
    & Gaussian & 27.88 & 0.58 & 17.73 & 6.03 & 71.83 \\
    & Golden & \textbf{28.08} & \textbf{0.76} & \textbf{17.90} & 6.05 & \textbf{73.87} \\
    & \textbf{Oracle (Ours)} & 27.97 & 0.67 & 17.85 & \textbf{6.13} & 73.00 \\
    \midrule
    
    \multirow{3}{*}{DrawBench} 
    & Gaussian & 28.07 & 0.51 & 19.91 & 5.43 & 69.79 \\
    & Golden & \textbf{28.49} & \textbf{0.58} & 19.98 & 5.44 & \textbf{70.55} \\
    & \textbf{Oracle (Ours)} & 28.25 & 0.57 & \textbf{20.02} & \textbf{5.48} & 70.21 \\
    
    \bottomrule
  \end{tabular}
\end{table}
\subsection{Comparison with Golden Noise}
We compare Oracle Noise with Golden Noise \cite{zhou2025goldennoisediffusionmodels}, a learned noise prior, in Table~\ref{tab:golden_vs_oracle_updated}. While Golden Noise sets an upper bound for semantic alignment (CLIP Score), it often induces subtle distribution shifts due to its aggressive optimization. In contrast, Oracle Noise consistently achieves superior Aesthetics scores (6.13 on Pick-a-Pic) across all benchmarks.

This supports our theoretical claim that strict spherical geodesic updates preserve the native Gaussian prior, preventing the over-saturation and artifacts typical of learned priors. Notably, Oracle Noise delivers these results in a zero-shot, training-free manner, offering a more efficient and prior-preserving alternative for real-world inference.

\begin{table}[H]
\centering
\caption{Ablation of CFG-Aware optimization on Pickapic under SDXL.}
\label{tab:cfg_aware_ablation}
\small
\setlength{\tabcolsep}{4pt} 
\begin{tabular}{lccccc}
\toprule
Method & HPS$\uparrow$ & IR$\uparrow$ & Pick$\uparrow$ & Aes.$\uparrow$ & CLIP$\uparrow$ \\
\midrule
Gaussian & 27.88 & 0.58 & 17.73 & 6.03 & 71.83 \\
Oracle (Cond) & 27.86 & 0.55 & 17.76 & 6.08 & 70.91 \\
\textbf{Oracle (CFG aware)} & \textbf{27.97} & \textbf{0.67} & \textbf{17.85} & \textbf{6.13} & \textbf{73.00} \\
\bottomrule
\end{tabular}
\end{table}

\subsection{Impact of CFG-Aware Optimization}The results in Table~\ref{tab:cfg_aware_ablation} reveal a critical optimization-inference mismatch in noise initialization. At a standard guidance scale ($s=7.5$), the \textit{Cond-only} variant surprisingly performs worse than the \textit{Gaussian} baseline in HPSv2 and CLIP Score. This degradation occurs because CFG inference relies on the contrast between conditional and unconditional branches; by only optimizing the former, the latent fails to account for background noise amplification during extrapolation, leading to over-saturation. In contrast, our \textbf{CFG-Aware} objective simulates the full inference dynamics during optimization, resolving this mismatch.As shown, it not only recovers the performance loss but significantly boosts the CLIP Score to 73.00 and ImageReward to 0.67. This demonstrates that CFG-awareness is essential for ensuring that pre-reverse-diffusion alignment remains robust under high guidance scales.

\section{Visualization of Multi-Encoder Token Weighting}
To provide a more intuitive understanding of the semantic routing mechanism introduced in Section~\ref{subsec:token_weighting}, Figure~\ref{fig:token_weighting_appendix} illustrates the step-by-step process of computing the token importance scores based on representational collapse. 

\begin{figure}[t]
\centering

\begin{subfigure}[t]{0.49\linewidth}
\vspace{0pt}
\centering
\includegraphics[height=0.30\textheight,width=\linewidth,keepaspectratio]{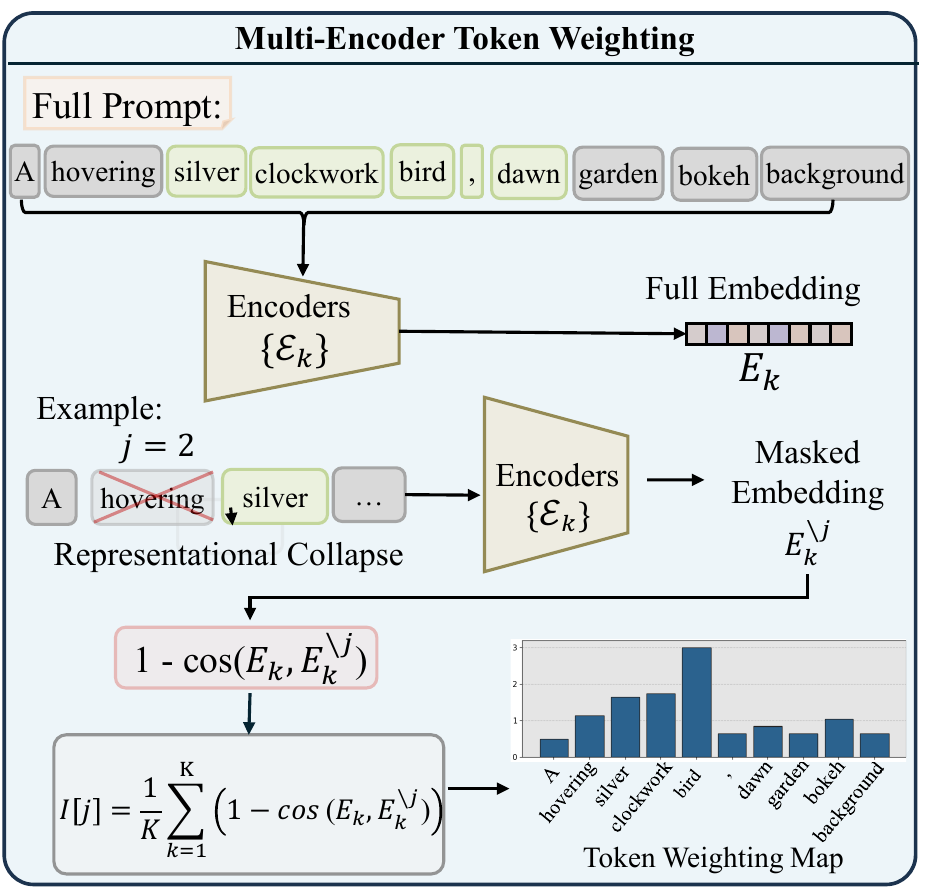}
\caption{Multi-Encoder Token Weighting.}
\label{fig:token_weighting_appendix}
\end{subfigure}
\hfill
\begin{subfigure}[t]{0.49\linewidth}
\vspace{0pt}
\centering
\includegraphics[height=0.30\textheight,width=\linewidth,keepaspectratio]{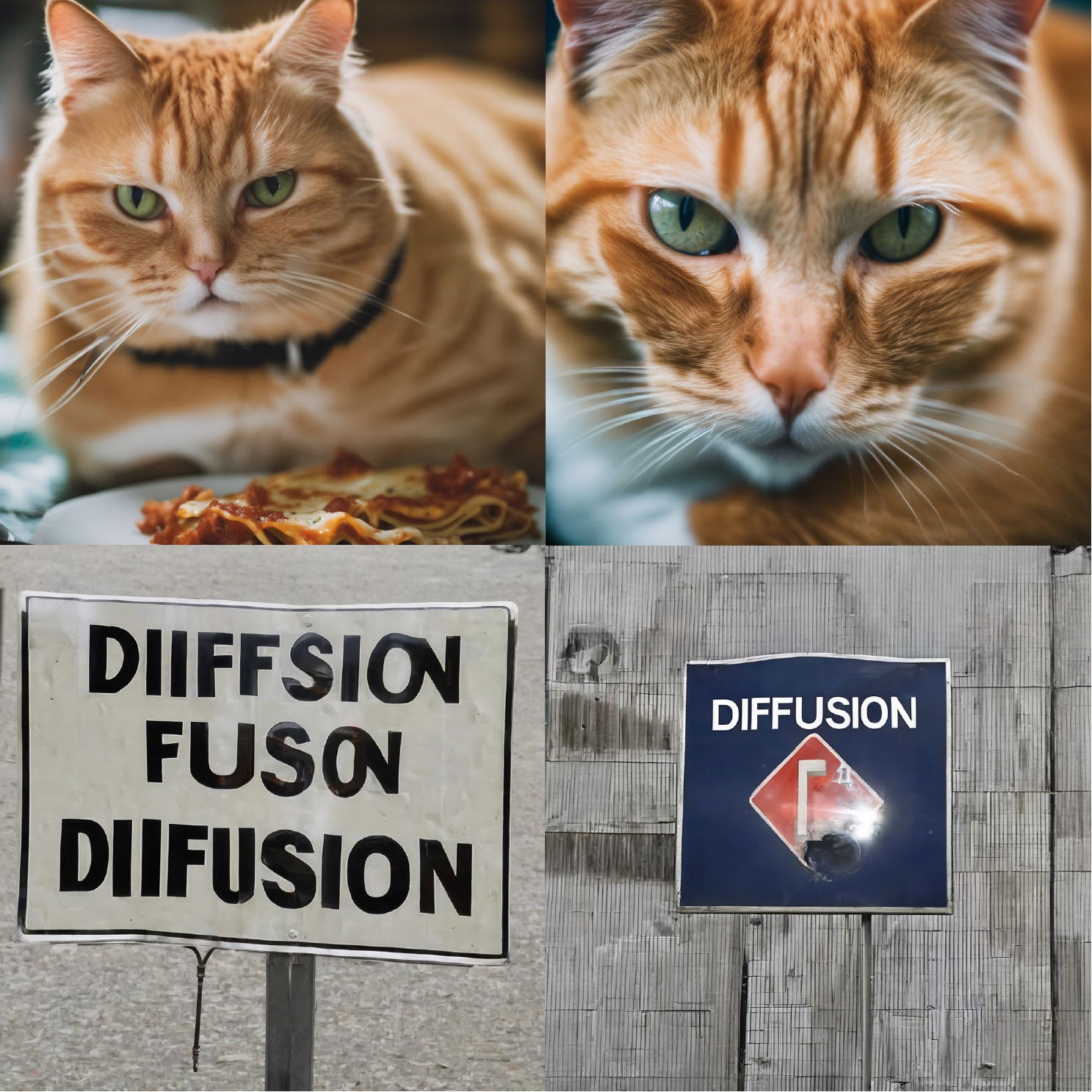}
\caption{Qualitative failure case.}
\label{fig:failure_case}
\end{subfigure}

\vspace{-0.4em}
\caption{\textbf{Analysis of token weighting and failure modes.}
Left: Multi-Encoder Token Weighting identifies high-value visual entities by masking individual tokens and measuring the embedding change, assigning higher optimization weights to core entities such as ``bird.'' Right: qualitative failure cases of Oracle Noise. The top row shows semantic suppression, where over-emphasizing the dominant ``cat'' token weakens or removes the secondary ``lasagna'' object. The bottom row illustrates text-rendering failures, where generated typography remains brittle and may exhibit misspellings or distorted layouts despite improved visual realism.}
\label{fig:appendix_analysis}

\end{figure}

\section{Failure Case Analysis: Semantic Suppression}

The contrast in Figure~\ref{fig:failure_case} reveals a critical boundary of our framework. While the Gaussian baseline maintains a balanced (though low-fidelity) layout, Oracle Noise may occasionally over-route optimization energy toward a dominant entity. 

A logical question arises: why does the same prompt yield successful results in Figure 4 but exhibit semantic suppression here? We attribute this variance to the \textbf{stochastic nature of noise initialization} and architectural sensitivity. In the successful cases of Figure 4, the initial random seed $z_T$ resided in a region of the latent manifold where the gradients for "cat" and "lasagna" were not mutually exclusive, allowing for balanced co-existence. However, in Figure~\ref{fig:failure_case}, the chosen seed sits in a trajectory where the "cat" token’s representational energy is disproportionately high. This creates an \textbf{"attention sink"} effect: the optimization process finds a much sharper gradient for the dominant subject, causing it to converge into a state that satisfies the primary objective so aggressively that the weaker "lasagna" signals are washed out. This comparison proves that while Oracle Noise is a powerful tool for structural alignment, its final output remains an interplay between the initial seed and the inherent biases of the text encoder.

\section{Extended Theoretical Analysis and Proofs}
\label{sec:appendix_theory}

In this section, we provide the mathematical formulations for the latent optimization dynamics. We establish the geometric necessity of Riemannian optimization by proving the orthogonality of scale-invariant gradients, deriving the exact exponential map on the hypersphere, bounding the Wasserstein distance for prior equivalence, and characterizing the geodesic overshoot via Taylor expansion.

\subsection{Gradient Orthogonality}
\label{app:gradient_orthogonality}

We first define the objective function $\mathcal{L}$ under the structural normalization constraints of the diffusion architecture \cite{rombach2022highresolutionimagesynthesislatent}.

\begin{lemma}
\label{lemma:orthogonality}
Let $\mathcal{L}: \mathbb{R}^D \setminus \{0\} \to \mathbb{R}$ be an objective function. If $\mathcal{L}$ is composed strictly of linear transformations followed by zero-mean normalizations (e.g., LayerNorm or GroupNorm), then $\mathcal{L}$ is a degree-zero homogeneous function. Consequently, its gradient $\nabla \mathcal{L}(z)$ is orthogonal to the input $z$:
\begin{equation}
    \langle z, \nabla \mathcal{L}(z) \rangle = 0
\end{equation}
\end{lemma}

\begin{proof}
Let $f(z) = Wz$ be the linear feature transformation. For any scalar $c > 0$, the normalization operator $N(\cdot)$ acts as follows:
\begin{equation*}
    N(cf(z)) = \frac{cf(z) - \mu(cf(z))}{\sigma(cf(z))} = \frac{c(f(z) - \mu(f(z)))}{c\sigma(f(z))} = N(f(z))
\end{equation*}
Since $\mathcal{L}$ is defined upon $N(f(z))$, it holds that $\mathcal{L}(cz) = \mathcal{L}(z)$ for all $c > 0$. Differentiating both sides with respect to $c$ using the chain rule yields:
\begin{equation}
    \frac{d}{dc} \mathcal{L}(cz) = \nabla \mathcal{L}(cz)^T z
\end{equation}
Evaluating this derivative at $c = 1$, since $\mathcal{L}(cz)$ is constant with respect to $c$, we obtain:
\begin{equation}
    \left. \frac{d}{dc} \mathcal{L}(cz) \right|_{c=1} = \nabla \mathcal{L}(z)^T z = 0
\end{equation}
Thus, $\langle z, \nabla \mathcal{L}(z) \rangle = 0$.
\end{proof}

\begin{theorem}
\label{thm:euclidean_degradation_app}
An unconstrained Euclidean gradient ascent step $z_{k+1} = z_k + \eta g$ strictly inflates the $\ell_2$-norm of the latent vector, causing divergence from the initial hypersphere \cite{Vershynin_2018}.
\end{theorem}

\begin{proof}
The squared $\ell_2$-norm of the updated latent is:
\begin{equation}
    \|z_{k+1}\|^2 = \|z_k + \eta g\|^2 = \|z_k\|^2 + 2\eta \langle z_k, g \rangle + \eta^2 \|g\|^2
\end{equation}
By Lemma~\ref{lemma:orthogonality}, $\langle z_k, g \rangle = 0$. We have:
\begin{equation}
    \|z_{k+1}\|^2 = \|z_k\|^2 + \eta^2 \|g\|^2 > \|z_k\|^2
\end{equation}
Therefore, the latent strictly escapes its native manifold $\mathbb{S}^{D-1}(\|z_0\|)$.
\end{proof}

\subsection{Riemannian Exponential Map}
\label{app:riemannian_exp_map}

To preserve the latent norm, optimization must be restricted to the Riemannian manifold $\mathcal{M} = \mathbb{S}^{D-1}(R)$, where $R = \|z_k\|$.

Let $T_{z_k} \mathcal{M} = \{v \in \mathbb{R}^D : \langle z_k, v \rangle = 0\}$ be the tangent space at $z_k$. The projection of the Euclidean gradient $g$ onto $T_{z_k} \mathcal{M}$ is:
\begin{equation}
    g_\perp = g - \frac{\langle z_k, g \rangle}{R^2} z_k
\end{equation}
\textit{(Note: While Lemma~\ref{lemma:orthogonality} implies $g_\perp = g$, we retain the projection for strictness under minor architectural residuals).}

\paragraph{Derivation of the Geodesic Step.} The steepest ascent curve starting from $z_k$ with velocity $v \in T_{z_k} \mathcal{M}$ is defined by the exponential map $\exp_{z_k}(v)$. On a hypersphere, geodesics are great circles. The great circle passing through $z_k$ in the direction of $v$ lies in the 2D plane spanned by the orthogonal vectors $z_k$ and $v$.

Let $u = \frac{v}{\|v\|}$ be the unit tangent vector. The curve parameterized by arc length $s$ is:
\begin{equation}
    \gamma(s) = z_k \cos\left(\frac{s}{R}\right) + R u \sin\left(\frac{s}{R}\right)
\end{equation}
Setting the ascent vector $v = \eta g_\perp$ and substituting $s = \|v\| = \eta \|g_\perp\|$, we obtain the exact update rule:
\begin{equation}
    z_{k+1} = \exp_{z_k}(\eta g_\perp) = z_k \cos\left( \eta \frac{\|g_\perp\|}{R} \right) + R \frac{g_\perp}{\|g_\perp\|} \sin\left( \eta \frac{\|g_\perp\|}{R} \right)
\end{equation}

\subsection{Prior Equivalence}
\label{app:wasserstein_equivalence}

We prove that restricting the distribution to the hypersphere introduces negligible distribution shift from the standard Gaussian prior $\mathcal{P}_G = \mathcal{N}(\mathbf{0}, \mathbf{I}_D)$.

\begin{theorem}
\label{thm:wasserstein}
Let $\mathcal{P}_S = \mathcal{U}(\mathbb{S}^{D-1}(\sqrt{D}))$ be the uniform distribution on the hypersphere. As $D \to \infty$, the 2-Wasserstein distance $\mathcal{W}_2(\mathcal{P}_G, \mathcal{P}_S)$ asymptotically vanishes relative to the expected norm \cite{Vershynin_2018, davidson2022hypersphericalvariationalautoencoders}.
\end{theorem}

\begin{proof}
Let $Z \sim \mathcal{P}_G$. We can decouple $Z$ as $Z = R_G \Theta$, where $R_G^2 \sim \chi^2(D)$ and $\Theta \sim \mathcal{U}(\mathbb{S}^{D-1}(1))$. We define a deterministic transport map $T: \mathcal{P}_G \to \mathcal{P}_S$ as $T(Z) = \sqrt{D} \Theta$.

\noindent The squared 2-Wasserstein distance is bounded by the cost:
\begin{equation}
    \mathcal{W}_2^2(\mathcal{P}_G, \mathcal{P}_S) \le \mathbb{E}_{Z \sim \mathcal{P}_G} \left[ \|Z - T(Z)\|^2 \right] = \mathbb{E}_{R_G} \left[ (R_G - \sqrt{D})^2 \right]
\end{equation}
Expanding the quadratic term:
\begin{equation}
    \mathbb{E}[(R_G - \sqrt{D})^2] = \mathbb{E}[R_G^2] - 2\sqrt{D}\mathbb{E}[R_G] + D
\end{equation}
By the properties of the Chi-squared distribution, $\mathbb{E}[R_G^2] = D$. The expectation of the Chi distribution is $\mathbb{E}[R_G] = \sqrt{2} \frac{\Gamma((D+1)/2)}{\Gamma(D/2)}$. Using the asymptotic expansion of the Gamma function (Stirling's series) for large $D$:
\begin{equation}
    \mathbb{E}[R_G] = \sqrt{D} \left( 1 - \frac{1}{4D} + \mathcal{O}(D^{-2}) \right)
\end{equation}
Substituting this back into the transport cost:
\begin{equation*}
    \mathcal{W}_2^2(\mathcal{P}_G, \mathcal{P}_S) \le 2D - 2\sqrt{D} \left[ \sqrt{D} - \frac{1}{4\sqrt{D}} + \mathcal{O}(D^{-3/2}) \right] = \frac{1}{2} + \mathcal{O}(D^{-1})
\end{equation*}
Evaluating the relative geometric deviation:
\begin{equation}
    \lim_{D \to \infty} \frac{\mathcal{W}_2(\mathcal{P}_G, \mathcal{P}_S)}{\mathbb{E}[\|Z\|]} \le \lim_{D \to \infty} \frac{\sqrt{1/2}}{\sqrt{D}} = 0
\end{equation}
Thus, $\mathcal{P}_S$ is asymptotically isomorphic to $\mathcal{P}_G$ in high dimensions.
\end{proof}

\subsection{Geodesic Overshoot}
\label{app:geodesic_overshoot}

We mathematically characterize the degradation of the objective function under excessively large angular step sizes $\eta$.

Let $\gamma(\eta) = z_0 \cos \eta + u \sin \eta$ denote the unit-speed geodesic, where $u = \|z_0\| \frac{g_\perp}{\|g_\perp\|}$. We evaluate the objective function $\mathcal{L}$ along $\gamma$ using a second-order Taylor expansion around $\eta = 0$:
\begin{equation}
    \gamma(\eta) = z_0 + \eta u - \frac{\eta^2}{2} z_0 + \mathcal{O}(\eta^3)
\end{equation}
Expanding $\mathcal{L}(\gamma(\eta))$ yields:
\begin{equation}
\begin{split}
    \mathcal{L}(\gamma(\eta)) &= \mathcal{L}(z_0) + \nabla \mathcal{L}(z_0)^T (\gamma(\eta) - z_0) \\
    &\quad + \frac{1}{2} (\gamma(\eta) - z_0)^T \mathbf{H}_{\mathcal{L}} (\gamma(\eta) - z_0) + \mathcal{O}(\eta^3)
\end{split}
\end{equation}
where $\mathbf{H}_{\mathcal{L}}$ is the Hessian of $\mathcal{L}$ evaluated at $z_0$. Substituting the expansion of $\gamma(\eta)$:
\begin{equation}
    \mathcal{L}(\gamma(\eta)) \approx \mathcal{L}(z_0) + \nabla \mathcal{L}(z_0)^T \left( \eta u - \frac{\eta^2}{2} z_0 \right) + \frac{1}{2} (\eta u)^T \mathbf{H}_{\mathcal{L}} (\eta u)
\end{equation}
By Lemma~\ref{lemma:orthogonality}, $\nabla \mathcal{L}(z_0)^T z_0 = 0$. The expansion simplifies  to:
\begin{equation}
    \mathcal{L}(\gamma(\eta)) \approx \mathcal{L}(z_0) + \eta \langle g, u \rangle + \frac{\eta^2}{2} u^T \mathbf{H}_{\mathcal{L}} u
\end{equation}

This expansion reveals the fundamental dynamics of the geodesic step. Because the update direction $u$ is constructed to be parallel to $g_\perp$, the first-order term $\eta \langle g, u \rangle > 0$ provides a strict positive gain that continuously drives the optimization forward. However, as the latent trajectory approaches an optimal semantic basin (i.e., a local maximum), the objective landscape becomes concave. In this region, the Hessian $\mathbf{H}_{\mathcal{L}}$ becomes negative definite along $u$, causing the second-order term $u^T \mathbf{H}_{\mathcal{L}} u < 0$ to emerge as a curvature-induced penalty that fundamentally resists the gradient ascent.

The phenomenon of ``geodesic overshoot'' occurs when the step size $\eta$ exceeds the critical threshold where the curvature penalty eclipses the gradient gain:
\begin{equation}
    \eta > -\frac{2 \langle g, u \rangle}{u^T \mathbf{H}_{\mathcal{L}} u}
\end{equation}
Beyond this threshold, the optimization trajectory physically wraps around the manifold's curvature, pointing the update vector away from the target semantic basin, necessitating the empirical bound of $\eta$ within the linear regime.

\section{Explanations of Models, Metrics, and Datasets}\subsection{Diffusion Models}\begin{itemize}\item \textbf{Stable Diffusion XL (SDXL) \cite{podell2023sdxlimprovinglatentdiffusion}:} An advanced iteration of the Stable Diffusion architecture featuring a significantly larger UNet backbone and a dual text-encoder setup (combining CLIP ViT-L and ViT-G). It enables high-resolution generation and demonstrates vastly improved prompt adherence and compositional capabilities.\item \textbf{SDXL-Turbo \cite{sauer2023adversarialdiffusiondistillation}:} A distilled, highly efficient variant of SDXL that leverages Adversarial Diffusion Distillation (ADD). It is designed to produce high-quality text-to-image outputs in a single or very few denoising steps.\item \textbf{Stable Diffusion 3.5 Medium (SD3.5-M) \cite{esser2024scalingrectifiedflowtransformers}:} A recent generative model utilizing a Multimodal Diffusion Transformer (MMDiT) architecture paired with rectified flow matching. It is optimized to balance high-quality semantic alignment with accessibility for consumer-grade hardware.\end{itemize}\subsection{Evaluation Metrics}\begin{itemize}\item \textbf{FID (Fr'echet Inception Distance) \cite{Seitzer2020FID}:} A standard quantitative metric evaluating overall image quality and distributional similarity. It measures the Wasserstein-2 distance between the feature representations of generated images and a real reference dataset. Lower scores indicate higher realism.\item \textbf{Vendi Score \cite{friedman2023vendiscorediversityevaluation}:} An evaluation metric designed to quantify the diversity of a generated set of images. It estimates the effective number of distinct modes in the generated distribution.\item \textbf{CLIP Score \cite{radford2021learningtransferablevisualmodels}:} A metric measuring the semantic alignment between a generated image and its corresponding text prompt. It is computed as the cosine similarity between the image and text embeddings.\item \textbf{HPSv2 (Human Preference Score v2) \cite{wu2023human}:} A specialized reward model trained on a large-scale dataset of human choices. It predicts human preference regarding text-image semantic alignment and visual appeal.\item \textbf{ImageReward \cite{xu2023imagerewardlearningevaluatinghuman}:} A general-purpose text-to-image human preference reward model trained on expert annotations. It addresses issues like anatomical correctness and artifacts.\item \textbf{PickScore \cite{kirstain2023pickapicopendatasetuser}:} An evaluation metric based on the Pick-a-Pic dataset, trained to predict human choices in side-by-side generative comparisons.\item \textbf{Aesthetics Score \cite{schuhmann2022laion5bopenlargescaledataset}:} A metric typically utilizing a linear classifier trained on CLIP embeddings from human-rated datasets. It predicts the visual and artistic appeal of an image.\end{itemize}\subsection{Benchmarks and Datasets}\begin{itemize}\item \textbf{MS-COCO 2017 \cite{lin2015microsoftcococommonobjects}:} A large-scale, widely-used dataset originally designed for object detection and captioning. In generative research, it is standardly used to test zero-shot generation fidelity.\item \textbf{DrawBench \cite{saharia2022photorealistic}:} A challenging, curated benchmark consisting of complex text prompts. It is designed to stress-test compositionality, spatial relations, and text rendering.\item \textbf{GenEval \cite{ghosh2023genevalobjectfocusedframeworkevaluating}:} An object-focused evaluation framework and dataset for testing fine-grained compositional reasoning in generative models.\item \textbf{Pick-a-Pic \cite{kirstain2023pickapicopendatasetuser}:} A large-scale, open dataset of human preferences for text-to-image generation collected via crowdsourced selections.\end{itemize}\section{Limitations and Future Work}\label{sec:limitations}While our proposed Oracle Noise framework effectively addresses the geometric degradation and semantic misallocation in latent initialization, it is not without limitations. Currently, the optimal optimization hyperparameters---specifically the angular step size $\eta$ and the total number of iterations $N$---are highly dependent on the semantic complexity of the input prompt. We empirically observe that applying fixed hyperparameter configurations across all prompts can lead to sub-optimal results. For instance, structurally simple prompts require minimal optimization energy; applying a large $\eta$ or $N$ in these cases often induces ``geodesic overshoot,'' where the latent trajectory physically bypasses the optimal semantic basin and introduces unnatural artifacts or degrades structural integrity. Conversely, highly complex prompts involving multiple subjects and intricate spatial relationships demand prolonged, fine-grained refinement, making a conservative static setting insufficient for perfect semantic alignment.Therefore, a highly promising avenue for future work is the development of an adaptive, prompt-aware scheduling mechanism. Rather than relying on fixed parameters, future iterations of this framework could dynamically modulate the step size and iteration count based on real-time optimization feedback, such as attention entropy or the variance of the projected gradients on the hypersphere. Additionally, integrating an early-stopping criterion guided by an internal convergence metric could effectively prevent over-optimization without the overhead of external proxy models. Finally, we plan to extend the principles of prior-preserving spherical optimization beyond text-to-image synthesis, exploring its potential applications in enhancing temporal consistency for video diffusion models and structural conditioning in 3D generative frameworks.

\bibliography{software}
\bibliographystyle{plainnat}
\end{document}